\def\year{2017}
\documentclass[letterpaper]{article}
\usepackage{aaai}
\usepackage{times}
\usepackage{helvet}
\usepackage{courier}

\usepackage{multicol}
\usepackage{multirow}

\usepackage{booktabs}
\usepackage{algorithm}
\usepackage{setspace}
\usepackage{algpseudocode}
\usepackage{amsmath}
\usepackage{amsthm}
\usepackage{bm}
\usepackage{amssymb}
\usepackage{graphicx}
\usepackage{mathtools}
\usepackage{tikz}
\usepackage{times}
\usepackage{subfigure}
\usetikzlibrary{calc,shapes}

\newtheorem{theorem}{Theorem}

\begin{document}
%

\title{Coping with Large Traffic Volumes in Schedule-Driven Traffic Signal Control}

\author{Hsu-Chieh Hu\\
Department of Electrical and Computer Engineering \\
Carnegie Mellon University, Pittsburgh, PA, USA\\
\texttt{hsuchieh@andrew.cmu.edu}
\And
Stephen F. Smith\\
The Robotics Institute,\\
 Carnegie Mellon University, Pittsburgh, PA, USA\\
\texttt{sfs@cs.cmu.edu}
 }

\maketitle
\begin{abstract}

Recent work in decentralized, schedule-driven traffic control has demonstrated the ability to significantly improve traffic flow efficiency in complex urban road networks. However, in situations where vehicle volumes increase to the point that the physical capacity of a road network reaches or exceeds saturation, it has been observed that the effectiveness of a schedule-driven approach begins to degrade, leading to progressively higher network congestion. In essence, the traffic control problem becomes less of a scheduling problem and more of a queue management problem in this circumstance. In this paper we propose a composite approach to real-time traffic control that uses sensed information on queue lengths to influence scheduling decisions and gracefully shift the signal control strategy to queue management in high volume/high congestion settings. Specifically, queue-length information is used to establish weights for the sensed vehicle clusters that must be scheduled through a given intersection at any point, and hence bias the wait time minimization calculation. To compute these weights, we develop a model in which successive movement phases are viewed as different states of an Ising model, and parameters quantify strength of interactions. To ensure scalability, queue information is only exchanged between direct neighbors and the asynchronous nature of local intersection scheduling is preserved. We demonstrate the potential of the approach through microscopic traffic simulation of a real-world road network, showing a $60\%$ reduction in average wait times over the baseline schedule-driven approach in heavy traffic scenarios. We also report initial field test results, which show the ability to reduce queues during heavy traffic periods.

\end{abstract} 

\section{Introduction}
Traffic congestion in urban areas is a serious problem, resulting in significant economic cost to drivers in terms of wasted time and fuel, as well as substantial environmental cost due to increased vehicle emissions. The problem is also quite challenging, requiring tight coordination among traffic signals across road networks with multiple, conflicting dominant flows that typically shift through the day. While it is generally believed that improvements to traffic signal control offer the most promise for reducing congestion, there is little consensus on how to solve this network-level coordination problem. Offline generation of fixed, coordinated timing plans optimizes for the average case (which can be quite different than actual traffic flows at any time) and also immediately start to ``age". However, conventional wisdom has also tended to argue against the viability of real-time adaptive signal control in complex urban environments. Centralized coordination mechanisms that adjust signal timing parameters (e.g., cycle time, green time splits and offsets) dynamically based on sensed traffic have had the most success \cite{Robertson1991,Lowrie1992,heung2005coordinated,gettman2007data}. But these approaches are inherently susceptible to scalability issues, and are generally designed for gradual adjustment (e.g., every 5-10 minutes) as opposed to real time adaptive signal control. On the other hand, decentralized, online planning approaches \cite{sen1997controlled,gartner2002optimized,shelby2001design,cai2009adaptive,jonsson2011scaling} have 
historically had difficulty in planning with a sufficiently long horizon to provide an effective basis for coordination.

Recent work in decentralized, online planning, however, has developed a schedule-driven approach to real-time traffic control that overcomes this horizon problem \cite{Xie2012,xie2012schedule}. Key to this approach is a formulation of the core intersection scheduling problem as a single machine scheduling problem, where input jobs are clusters of vehicles in close proximity to each other  (i.e., approaching platoons, queues). This aggregate representation allows plans to be efficiently generated with order-of-magnitude longer horizons than was previously possible, and hence enables network-level coordination through exchange of schedule information. Under this approach, the goal is to allocate green time to different signal  {\emph phases}, over time, where a signal phase is a compatible traffic movement pattern (e.g., East-West traffic flow). Each intersection asynchronously computes a schedule of green phases that minimizes the cumulative delay through the intersection of all approaching vehicles, and then communicates expected outflows to its downstream neighbors as it begins to execute its schedule. Scalability is ensured by the fact that intersections only communicate with their direct neighbors. However, since the planning horizon is extended, outflow information can propagate to non-local neighbors. Results obtained in an initial field test showed a 25\% reduction in travel times, a 40\% reduction in wait times and 30\% reduction in number of stops through the network \cite{smith2013smart}, and the system currently controls a network of 50 intersections in the East End area of Pittsburgh PA.


Despite these results, however, there are circumstances when the effectiveness of such a schedule-driven approach to coordination can break down. In particular in situations of high congestion, where the number of vehicles approaches the 
physical capacity of interconnected queues in the network, the traffic control problem becomes less of a scheduling problem (e.g., involving just a single cluster spanning the planning horizon along each intersection approach in the extreme case), and more of a problem of managing queues. In this paper, we propose a composite approach to real-time traffic control that addresses this issue, by using sensed information on queue lengths to influence scheduling decisions and gracefully shift the signal control strategy to queue management in high volume/high congestion settings.

To stabilize queues in a network (i.e., to prevent unbounded growth of queues), the vehicle clusters associated with longer queues should be serviced first. Within the above intersection scheduling framework, one straightforward way of achieving this behavior is to assign higher weights (i.e., higher priority) to these input clusters and compute phase schedules that minimize weighted cumulative delay. To balance the emphasis placed on queue management as a function of network saturation, we propose to use queue-length information (both local to the intersection and non-local from neighbors) to establish the weights. In situations where queue lengths are small, cluster priority will continue to be a function of the cluster size (number of vehicles) as before; however as the network becomes saturated and queues become longer, clusters associated with longer queues will begin to dominate cluster priority. To ensure scalability, queue information is only exchanged between direct neighbors and the asynchronous nature of local intersection scheduling is preserved. 

To derive an appropriate set of weights, the signal phases of a given intersection are viewed as different states of an Ising model \cite{suzuki2013chaotic}  and the probabilistic distribution of this model, whose parameters quantify transitions between phases and strength of interactions in terms of queue-length information, is calculated. However, computing the exact distribution is a hard problem \cite{cipra2000ising}. Hence we turn to approximation through mean field methods originating in statistical physics and the graphical model literature \cite{wainwright2008graphical}. The marginal distribution derived for each phase is then used as the weight of that phase's clusters.
We show formally that the proposed composite approach prevents queues from increasing without bound and therefore achieves network stability. We also present simulation results on a real-world traffic network that demonstrate the ability of our approach to effectively integrate queue management into schedule-driven traffic control. The approach is shown to reduce average waiting times by $60\%$ in heavy traffic scenarios. Finally, we report the results from some initial experiments in the field, which verify the ability to reduce queues during heavy traffic periods.

The remainder of the paper is organized as follows. We first summarize the baseline schedule-driven approach and how we propose to extend it. Next, the mechanisms necessary to achieve decentralized queue management are discussed. Then, we review related work in queue management approaches to network congestion. Finally, an empirical analysis of the composite approach  is presented, and conclusions are drawn.

\section{Schedule-Driven Traffic Control}

As indicated above, the key to the single machine scheduling problem formulation of the schedule-driven approach of \cite{Xie2012,xie2012schedule} is an aggregate representation of traffic flows as sequences of clusters $c$ over the planning (or prediction) horizon. The clusters become the jobs that must be sequenced through the intersection (the single machine). The cluster sequences provide short-term variability of traffic flows at each intersection and preserve the non-uniform nature of real-time flows. Specifically, the input is an ordered sequence of $(|c|, arr, dep)$ triples reflecting each approaching or queued vehicle on each road segment, where  $|c|$, $arr$ and $dep$ are number of vehicles, arrival time and departure time respectively. The vehicles are sensed through the intersection's detectors.

Once the cluster sequences of each approaching road segments are represented, each cluster is viewed as a non-divisible job and a forward-recursion  dynamic programming search is performed to generate a phase schedule that minimizes the cumulative delay of all clusters in the current prediction horizon. The process constructs an optimal sequence of clusters that maintains the ordering of clusters along each road segment, and each time a phase change is implied by the sequence, then a delay corresponding to the intersection's yellow/all-red changeover time constraints is inserted.  If the resulting schedule is found to violate the maximum green time constraints for any phase (introduced to ensure fairness), then the first offending cluster in the schedule is split, and the problem is re-solved. 

More precisely, the delay that each cluster contributes to the cumulative delay $\sum_c d(c)$ is defined as 
\begin{equation}
d(c) = |c| \cdot (ast - arr(c)),
\end{equation} 
where $ast$ is the actual start time determined by the process through a forward recursion. The optimal sequence (schedule) is the one that incurs minimal delay for all vehicles.

As mentioned earlier,  our hypothesis is that the effectiveness of this schedule-driven process degrades as congestion increases near saturation, due to the fact that it becomes more and more difficult to accurately predict when incoming clusters are going to arrive at the intersection when queues become large. Note that a queueing network is considered to be \textit{stable} if the queues do not tend to increase without bound. To boost the performance of this schedule-driven process in a network that is experiencing high congestion, we introduce a weight into this delay computation. The basic idea is to bias the scheduling search more toward stabilizing local queues (both at the local intersection and at its neighbor intersections) as the level of local congestion increases. To measure the level of congestion, we rely on queue-length information associated with various phases. To provide a low complexity scheme for queue management, we propose to weight each cluster of a given phase equally. The delay incurred by each cluster is thus rewritten as
\begin{equation}
d(c) = |c| \cdot (ast - arr(c)) \cdot w(p),
\end{equation} 
where $w(p) $ is the weight assigned to the phase $p$ that cluster $c$ belongs to. The important question then becomes: how to set the weights for competing phases.

\section{Queue Management Using Mean Field Methods}
 
 In this section, we introduce a decentralized method for calculating the weights to assign to each phase of a given intersection, so that its queues are locally stabilized in coordination with the queues at neighboring intersections. In brief, we propose a special Ising model in which the phases of a given traffic signal are viewed as different states, and their respective queue lengths are its parameters. Then, the intractable marginal distribution of the phases in each intersection is approximated by appealing to the use of mean field methods. We use the marginal distribution as the weight (priority) of each phase.
 
 \subsection{Weight in a Form of Probability Function}
To stabilize queues in the network, the clusters associated with longer queues should be served with higher priority. Hence, the weight function representing priority should be a function of queue length. The goal is to find a proper function that can reflect such queue dynamics and is suitable for integration with schedule-driven traffic control. Let us assume that the sum of weights of all phases is equal to one,  
\begin{equation}
\sum_{p = 1}^{P} w(Q_p) = 1,
\end{equation}
where $Q_p$ is the queue length at the corresponding phase $p$ and $P$ is number of phases for a single intersection.

A probability function is a reasonable choice as the phase of each intersection is viewed as a random variable. A larger weight for a specific phase implies a higher probability of that phase occurring. Furthermore, the probability function is a continuous function of queue length matching the continuous variability of vehicular cluster size. 
In the following sections, we apply techniques from statistical physics and graphical models to derive this probability function.

\subsection{Boltzmann Distribution for Traffic Signal}
The green and red lights of a traffic signal can be viewed as two states of an Ising particle spin. Moreover, intersections in urban environments are interconnected and interact with each other. In statistical physics,  a Boltzmann distribution is a probability distribution that assigns probability that an interacting system is in a certain configuration of states described by an energy function of states. A standard form of Boltzmann distribution that has binary states and energy function $E(\cdot)$ is defined as follows:
\begin{equation}
p(\bm{\sigma}; \theta) = \frac{1}{Z( \theta)} e^{-\beta E(\bm{\sigma}; \theta)},
 \end{equation}
 where $\bm{\sigma} \in \{1,0\}^n$ are state variables, and $T = \frac{1}{\beta}$, and $Z( \theta)$ are temperature and normalization constants. We assume that $\beta$ is assimilated into the $\theta$ parameters. Consider a graph $G = (V,E)$, where $s \in V$ and $t \in V$ are two adjacent nodes (intersections) in $G$. Then the energy function incorporating interaction takes the following form,
\begin{equation}
E(\bm{\sigma}; \theta) = -\sum_{(s,t)\in E} \theta_{st} \sigma_s\sigma_t - \sum_{s \in V} \theta_s \sigma_s,
\end{equation}
where the first term defines the interaction between two intersections and the second term specifies the external field of each intersection. With this form of energy function, a Boltzmann distribution can also be expressed by the following general form of exponential family:
\begin{align}
&p(\bm{\sigma}; \theta) = \exp(\langle\sigma, \theta\rangle) - A(\theta)),
\end{align}
where $\langle \cdot \rangle$ is Euclidean inner product of two vectors and $\log Z(\theta) = A(\theta)$. After getting the distribution, the marginal probability of phase $p$ at intersection $i$ is calculated, and we set the weight for phase schedule generation to this probability, 
\begin{equation}
W(Q_{ip}) = P(\sigma_i = p).
\end{equation}

\subsection{Queue-based Energy Function}
In an interconnected queueing network, the strength of interaction $\theta_{st}$ and external potential $\theta_{s}$ are related to queue length, which corresponds to different phases (states). Suppose that we have two phases: $\{1\}$ represents East-West phase, and  $\{0\}$ represents North-South phase. Taking the East-West phase as an example, the external field $\theta_{s} >0 $ is the pressure that "pushes" vehicles along with East-West direction.   The stronger the external field, the greater the tendency of the vehicles to keep moving East-West.   $\theta_{s} <0$ is analogous to a repulsive field that prevents vehicles from approaching further.  Moreover,  the queues of neighbor intersections contributing to East-West phase will be summed up together and used to measure the interaction strength $\theta_{st}$. Specifically, $\theta_{st}$ is a measure of repulsion or attraction faced by intersections as they synchronize with another intersection. The sign of $\theta_{st}$ corresponding to North-South phase is opposite to that of the East-West phase, since the traffic flows in both directions compete with each other. The energy function of  a two-phased signalized transportation network is given by


\begin{align}
&E(\bm\sigma; \theta) =  -\sum_{(s,t)\in E} \theta_{st} \sigma_s\sigma_t - \sum_{s \in V} \theta_s \sigma_s\\
\nonumber & =  -\sum_{(s,t)\in E} L(Q_{t \rightarrow s} , Q_{s\rightarrow t})   \sigma_s\sigma_t - \sum_{s \in V} (Q_{s,h} - Q_{s,v}) \sigma_s,
\end{align} 
where $L(Q_{t \rightarrow s} , Q_{s\rightarrow t})$ is the interaction strength and depends on which phase these queues are contributing to. It is defined as  
\begin{align}
   L(Q_{t \rightarrow s} , Q_{s\rightarrow t}) = 
\begin{cases}
    Q_{t \rightarrow s} - Q_{s\rightarrow t},& \text{if } (s,t)\in E_h\\
    Q_{s\rightarrow t} - Q_{t \rightarrow s}, & \text{if } (s,t)\in E_v,
\end{cases}
\end{align}
where $h$,$v$ denote the East-West and North-South directions. $E_h$ and $E_v$ are the sets of the East-West and North-South road segments respectively. Specifically, the queues served during the East-West phase contribute to positive terms of the energy function, while  those served during the North-South phase contribute to negative terms. The intersection of the signalized transportation network is depicted in Figure~\ref{intersection}.

\subsection{Calculation of the Weights}
For a distribution associated with a complex graph, especially with loops that typify grid networks, it is intractable to perform probabilistic inference, e.g., compute the exact marginal distribution of all random variables. The variational approach to the probabilistic inference involves converting the inference problem into an optimization problem, by approximating the feasible set, and solving the relaxed problem.

A Boltzmann distribution is one of the exponential families \cite{wainwright2008graphical}. An appealing feature of the exponential family is that moments of the distribution are obtained by the derivatives of log normalization function $A(\theta)$. For a given tractable subgraph $F$, mean field methods are based on optimizing over the subset of realizable mean parameters $\mu$ that can be obtained by the subset of exponential family distribution. With the subset $\mathcal{M}_F(G)$ of $\mu$ and the corresponding conjugate dual function $A^*_F(\mu)$, the $A(\theta)$ can be computed by solving the following optimization problem

\begin{figure}[!htbp]
\centering
\includegraphics[scale = 0.37]{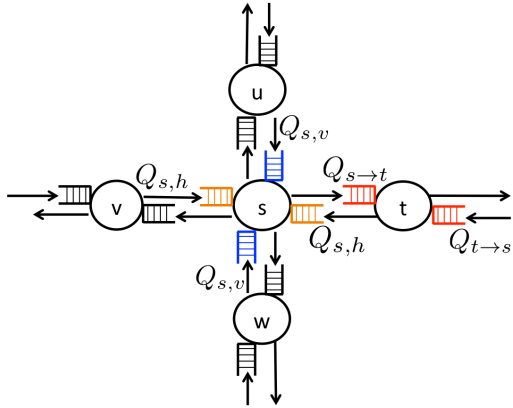}
\caption{The two-phased signalized intersection associated with a two-state Ising model}
\label{intersection}
\end{figure}

\begin{align}
 & A(\theta) = \max_{\mu\in \mathcal{M}_F(G)} \langle \mu, \theta \rangle  - A^*_F(\mu)
\end{align}
and the resulting mean parameter is 
\begin{align}
\mu_s = E_{\theta}[\sigma_s] = P(\sigma_s = 1;\theta)
\end{align}

In this work, the approximation is based on choosing product distribution 
\begin{equation}
p(\sigma_1, \sigma_2, \cdots, \sigma_n; \theta) = \prod_{s\in V} p(\sigma_s; \theta)
\end{equation}
 as the tractable approximation. It is also referred to as the naive mean field approach. According to this approximation, the optimization problem is rewritten as

\begin{align}
 \nonumber A(\theta) =&\max_{\mu\in[0,1]}    \sum_{s\in V}\sum_{t \in \mathcal{N}_s} L(Q_{t\rightarrow s}, Q_{s\rightarrow t})   \mu_s\mu_t \\
 \nonumber &+ \sum_{s \in V} (Q_{s,h} - Q_{s,v}) \mu_s\\
 & - \sum_{s \in V} [\mu_s \log (\mu_s) - (1-\mu_s) \log(1-\mu_s)]
\end{align}

Solving the problem yields a specific form of the mean parameter update 
\begin{align}
\nonumber \mu_s \leftarrow & \{1+ \exp[ -(Q_{s,h} - Q_{s,v}) - \sum_{t \in \mathcal{N}_s} L(Q_{t\rightarrow s}, Q_{s\rightarrow t}) \mu_t ]\}^{-1}\\
\nonumber  = &S\big( Q_{s,h} + \sum_{t \in \mathcal{N}_{s,h}} (Q_{t\rightarrow s}  - Q_{s\rightarrow t}) \mu_t \\
&- Q_{s,v} -\sum_{t \in \mathcal{N}_{s,v}} (Q_{t\rightarrow s}  - Q_{s\rightarrow t}) \mu_t  \big),
\label{update}
\end{align}
where $S(x)$ is sigmoid function $\frac{1}{1+\exp(-x)}$. $\mathcal{N}_{s,h}$ and $\mathcal{N}_{s,v}$ are sets of neighbor intersections corresponding to East-West and North-South phases. From (\ref{update}), the two terms in the sigmoid function denote effective queues for each phase, which are defined as
\begin{align}
 \nonumber&\hat{Q}_{s,h} = Q_{s,h} + \sum_{t \in \mathcal{N}_{s,h}} (Q_{t\rightarrow s}  - Q_{s\rightarrow t}) \mu_t \\
 &\hat{Q}_{s,v} = Q_{s,v} + \sum_{t \in \mathcal{N}_{s,v}} (Q_{t\rightarrow s}  - Q_{s\rightarrow t}) \mu_t 
 \label{hatq}
\end{align}

With the effective queues, the marginal distribution is expressed concisely as
\begin{align}
\nonumber&P(\sigma_s = 1) = S\big( \hat{Q}_{s,h}  - \hat{Q}_{s,v} )\\
&P(\sigma_s = 0)  = 1-S\big( \hat{Q}_{s,h}  - \hat{Q}_{s,v} ).
\label{probex}
\end{align}
Note that the marginal distribution is a Bernoulli distribution whose parameter is only related to the difference between the effective queues. In other words, the weight function used in scheduling is a function of queue difference.

\subsection{Theoretical Guarantees of Stability}
In this section, we prove that by applying this weight function to clusters, an upper bound on the expected queue length is achieved. According to Little's law \cite{little1961proof}, the delay is bounded as well. 

The weight formulas (\ref{hatq}) and (\ref{probex}) specify that each actual queue has an effective queue associated with it. Upstream queue length can be viewed as a prediction of future traffic flow.  Most importantly, the contribution of the downstream queue prevents the intersection from spillover effects \cite{daganzo1998queue} by reducing the effective queue to a negative value. It is equivalent to decreasing the service rate when the road segment has insufficient capacity. Building upon these results as well as \cite{tassiulas1992stability}, we state the following property of our algorithm.
\begin{theorem}
Consider a network has $n$ queues with arrival rates $\pi_1,\cdots, \pi_n$. Under the proposed schedule-driven traffic control, expected queue length is bounded
\begin{equation}
 \limsup_t E[\sum_{i = 1}^{n}   Q_i(t)] \leq \frac{n^2 }{2\epsilon}
 \label{bound}
 \end{equation}
 if for any queues the arrival rates satisfy $\pi_i \leq s_i-\epsilon$ with $\epsilon > 0$ and service rate $s_i$.
\end{theorem}
\begin{proof}
To establish the upper bound  for any scheduling algorithm, it is sufficient to consider a multi-hop network model. Let $Q_{d(i)}$ designate the downstream queue of $Q_{i}$. Using Lyapunov-Foster theory and 
separating all queues into two cases: 1) $Q_i(t) < Q_{d(i)}(t)$ and 2) $Q_i(t)  \geq Q_{d(i)}(t)$, we have
\begin{align}
\nonumber&\sum_{i = 1}^{n}   Q_i(t) \leq \frac{n^2 }{2\epsilon} \\
&+ \Big( \frac{n}{2\epsilon}\sum_{i:Q_i(t) < Q_{d(i)}(t) } (\pi_i(t) - s_i(t)) [Q_i(t) - Q_{d(i)}(t)] \Big)
\label{bound}
\end{align}
In (\ref{bound}), $s_i(t)$ will be close to zero when $Q_i(t) - Q_{d(i)}(t)$ contributes a negative value according to (\ref{hatq}). Therefore, the bound can be rewritten as 
\begin{equation*}
E[\sum_{i = 1}^{n}   Q_i(t)] < \frac{n^2 }{2\epsilon}
\end{equation*}

%
\end{proof}

\section{Coordination Mechanisms}
In this section, we propose a message-passing protocol for coordinating use of queue-length information based on the mean field method. Basically, the protocol aims to balance the queues of different phases through exchanging queue and mean parameter information. To deal with practical considerations, two modifications to accommodate multiple phases and turning proportions are first proposed.    

\subsection{Multiple Phases}
Until now we have assumed that each traffic signal has two phases respectively. Practically speaking, urban intersections frequently have more than two phases, e.g, left turn phases at a four-way intersection, to provide more degrees of freedom to vehicles. According to the previous sections, the final result of weight formula is a softmax function, with only two exponential terms in denominator. We can generalize the Ising model to adapt to multiple phases by defining  parameters associating with multiple phases.   
If (\ref{probex}) is rewritten as
\begin{align}
\nonumber& S\big( \hat{Q}_{s,h}  - \hat{Q}_{s,v} )= \frac{ \exp(\hat{Q}_{s,h})}{ \exp(\hat{Q}_{s,h}) + \exp( \hat{Q}_{s,v})}
\end{align}
and 
\begin{align}
\nonumber&1- S\big( \hat{Q}_{s,h}  - \hat{Q}_{s,v} )= \frac{ \exp(\hat{Q}_{s,v})}{ \exp(\hat{Q}_{s,h}) + \exp( \hat{Q}_{s,v})},
\end{align}
it shows that the probability of a specific phase is proportional to exponential function of the corresponding effective queue. Therefore, we can derive the probability of phase $p$ in $P$-phased signals
\begin{align}
&P(\sigma_s = p) = \frac{ \exp(\hat{Q}_{s,p})}{ \sum_{i= 1}^P\exp(\hat{Q}_{s,i}) }.
\label{multi}
\end{align}

\subsection{Consideration of Turning Proportions}

Considering turning proportions at each intersection is important for improving performance of adaptive traffic signal systems. In the baseline schedule-driven approach, the turning movement proportion is estimated by taking moving averages of traffic flow rate for different phases respectively. The lane detectors detect the numbers of turning vehicles, compute the moving average and then normalize these flow rates. After getting these proportions, the scheduled flow is able to reflect the realistic traffic flow by proportioning the add-on flow and evacuated flow. For a grid-like network, the three input queues (east, north, and west) of the upstream intersection multiplied by the corresponding turning proportions are summed up together to obtain the upstream effective queue $Q_{u\rightarrow s}$ to downstream input queue (north). Similarly, the three proportioned output queues of three downstream intersections (east, south, and west) constitute the downstream effective queue $Q_{s\rightarrow d}$. 
\begin{align}
\nonumber &Q_{u \rightarrow s} = \sum_{k = 1}^{E} \zeta_k Q_{u\rightarrow s}^{(k)}\\
& Q_{s\rightarrow d} = \sum_{k = 1}^{E} \eta_k Q_{s\rightarrow d}^{(k)},
\end{align}
where $\zeta_k$ and $\eta_k$ are the turning proportions of input and output queues and $E$ is the number of input or output flows.

\subsection{Message-passing Protocol}
We now introduce a practical protocol to manage the queues of schedule-driven traffic control. The protocol realizes the solution (\ref{hatq}) and (\ref{probex}) of the mean field method in a fully decentralized manner. Moreover, it takes the multiple phase case and turning proportions into consideration. We assume that each intersection knows its neighbor intersections and is able to communicate with them. First, the scheduling agent collects its local queue-length information. Once the queue-length information and the calculated mean parameter are received from the neighbor intersections, the agent then computes its mean parameters and applies them as its cluster weights for generating the phase schedule. The protocol is summarized as follows

\scalebox{0.9}{
\begin{minipage}{\columnwidth}

\begin{algorithm}[H]
\floatname{algorithm}{Message-passing Protocol}
\renewcommand{\thealgorithm}{}
\caption{Steps that intersection $s$ communicates with each other to stabilize queues of the network}
\label{protocol1}
\begin{algorithmic}[1]
\State For each intersection $s$, let $Q_{s,p}$ denote estimated local queue length of phase $p$.
\Statex
\State Send proportioned $Q_{s\rightarrow down}$ to downstream neighbors and non-proportioned $Q_{s\rightarrow up}$ to upstream neighbors.  
\Statex
\State Receive proportioned $Q_{up\rightarrow s}$ and $\mu_{up}$ from upstream intersections.
\Statex
\State  Receive $Q_{down\rightarrow s}$ and $\mu_{down}$ from 3 downstream connecting intersections. Calculate proportioned $Q_{down\rightarrow s}$ with estimated turning proportions.
\Statex
\State Update  $\hat{Q}_{s,p}$ and $P(\sigma_s = p)= \mu_s$ based on (\ref{hatq}) and (\ref{probex}). Share $\mu_s$ with neighbors.
\end{algorithmic}
\label{protocol}
\end{algorithm} 

\end{minipage}
}

This protocol only requires communication with  direct neighbors and queue-length information. It is more robust than methods that use estimated arrival rates, since the latter are vulnerable to traffic variability. Furthermore, the approach is fully distributed to ensure scalability. 


\section{Performance Evaluation}
To evaluate our approach, we simulate performance on a real world network with $2$-way, multiple lane, and multi-directional traffic flow. The network model is based on the Baum-Centre neighborhood of Pittsburgh, Pennsylvania as shown in Figure~\ref{surtracmap}. The network consists mainly of 2-phased intersections, with just three 3-phased intersections. All simulation runs were carried out according to a realistic traffic pattern from late afternoon through PM rush (4-6 PM). The traffic pattern ramps up volumes over the simulation interval as follows:  (0-30mins: 236 cars/hr, 30min-1hr: 354 cars/hr, 1hr-2hrs: 528 cars/hr ). This simulation model presents a complex practical application for verifying the effectiveness of the proposed approach.

\begin{figure}[!htbp]
\centering
\includegraphics[scale = 0.3]{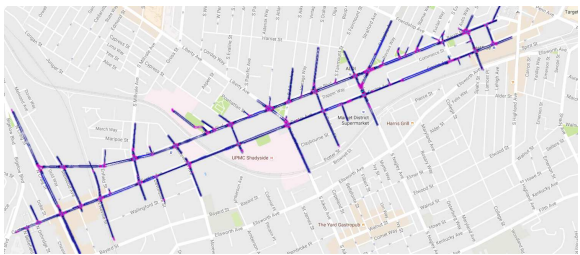}
\caption{Map of the 24 intersections in the Baum-Centre neighborhood of Pittsburgh, Pennsylvania}
\label{surtracmap}
\end{figure}

The simulation model was developed in VISSIM, a commercial microscopic traffic simulation software package that is well known in the transportation research community. To assess the performance boost provided by integrating the proposed queue management scheme with the original schedule-driven approach, we measure the average waiting time of all vehicles over 5 runs and take the performance of the original schedule-driven traffic control system as our benchmark (referred to as the baseline system below).

Table~\ref{resultstable2} shows the results of combining schedule-driven traffic control with two queue management schemes. The first scheme only uses local queue-length information to weight vehicle clusters. It is equivalent to the first term of right-hand side of (\ref{hatq}). The second scheme uses (\ref{hatq}) to derive weights. Compared to the benchmark, use of just local queue-length information reduces delay by $15\%$ and halves the standard deviation. The gain comes from the controlling the lengths of queues and decreasing cluster size. With the queue-length information and weights (mean parameters) from neighbors, delay is reduced by $24\%$. Use of this additional information achieves better performance since it avoids the spillover effect \cite{daganzo1998queue} by stopping vehicles further away from entry into a road segment with insufficient capacity. In addition, when upstream intersections send more traffic, the corresponding phase in the downstream intersection wills have longer green time to deal with it. 


\begin{table}[tp]
\centering
\caption{Avg. delay of Baum Centre Model}
\scalebox{0.9}{
\begin{tabular}{  l| c | c  } 
  & Mean  (s) &  Std. deviation    \\  \hline
 Benchmark & 124.27 & 103.87  \\  \hline
Local queue & 106.72 & 83.56 \\ \hline
Local + Upstream& 95.31 & 72.10 \\
Downstream queue & & 
\end{tabular}
}
\label{resultstable2}
\end{table}

\begin{table*}[tp]
  \centering
 \resizebox{1.15\columnwidth}{!}{%
  \begin{tabular}{*{9}{c}}
   \toprule
      \multirow{2}{*}{}& \multicolumn{4}{c}{ Queue length (no.)}& \multicolumn{4}{c}{Cluster size (msec)}  \\
    
   \cmidrule(l){2-5}  \cmidrule(l){6-9}    & \multicolumn{2}{c}{Benchmark} &  \multicolumn{2}{c}{Local+up+down}   &\multicolumn{2}{c}{ Benchmark} & \multicolumn{2}{c}{Local+up+down}  \\
   & mean  & std. & mean & std. & mean &std. & mean & std.\\
    \midrule
   Baum-Aiken & 33.60& 25.40& 6.42&7.18 & 24640& 30218 & 7174& 9601 \\
   Baum-Craig & 3.40&3.13 &2.98 & 2.91&7279 & 13448&4818& 5907\\ 
   Baum-Cypress & 3.02& 3.90 & 2.79 &3.41 &5940 &8115 &5412& 7252 \\
   Baum-Graham & 47.15 & 58.59 & 2.60 & 3.13&23775 &35934 &4680& 5331 \\
   Baum-Millvale & 10.80 & 11.72 & 6.72 & 5.30 & 8446 &12951 & 5618 & 6043\\	
   Baum-Liberty & 18.96 & 12.84 &15.97 & 11.72&14071 & 21687 &7423&10857\\
   Baum-Melwood & 10.88 & 13.97 &5.75 &5.44 &9193 & 15175 &5708 &6278\\
   Baum-Negley & 21.01 & 21.96 &7.10 & 5.79 &18593 & 29677 & 6142& 7762\\
   Baum-Roup & 38.32 & 48.48 &4.18 &4.53 &26550 & 42747&5494&6414 \\
   Centre-Aiken & 32.97 & 43.07&3.30 & 3.00& 9886 & 20864 & 3949& 6003\\
   Centre-Craig & 8.09 & 14.87 &5.75 & 7.22 &4453 &8387& 3566& 5223 \\
   Centre-Cypress & 2.68 &2.74 &2.83 &2.78 &6446 & 12992 &4252& 5115 \\
   Centre-Millvale & 6.17 &13.61 &2.23 & 2.36 &7238 & 12848 &4588& 5813\\
   Centre-Morewood & 4.89 & 4.72 &4.89 & 4.54 &6054 &8231 &5228& 6410\\
   Centre-Negley & 5.55 & 5.14 &5.68 &  4.14 &10306 & 24704 &3881& 4066 \\
   Centre-Neville & 5.80 & 13.29 &2.73 & 4.78 &7974 & 15888 &4515& 6522\\
    \bottomrule
  \end{tabular}
  }
   \caption{Queue length and cluster size of the intersections under high demand traffic}
  \label{demandtable}
\end{table*}


Since our approach focuses on highly congested scenarios, knowing the distribution of delay to vehicles helps us verify the effectiveness of the proposed queue management protocol. As shown in Figure~\ref{cdf}, the queue management scheme that applies local queue-length information only shifts the CDF leftward and provides $20\%$ improvement over benchmark for $90\%$ of the vehicles. Furthermore, with the addition of neighbor queue-length information, the improvement could be up to $30\%$ for $90\%$ of the vehicles. It should be noted that the queue management reduces average delay by $70s$. However, for the congested vehicles the reduction is more than $100s$. In other words, minimizing queue length is especially effective for high congestion scenarios. Likewise, we can also observe a similar phenomenon if we categorize traffic demand into three different groups: high (528 cars/hr), medium (354 cars/hr), and low (236 cars/hr).  Figure~\ref{three}, shows that improvement in delay can reach $60\%$ for the high traffic demand case. On the other hand, the performance of the medium and low traffic are comparable with the benchmark. Under medium and low traffic demand conditions, schedule-driven traffic control dominates performance, and the benefit of queue management is marginal. 

\begin{figure}[!htbp]
\centering
\includegraphics[scale = 0.35]{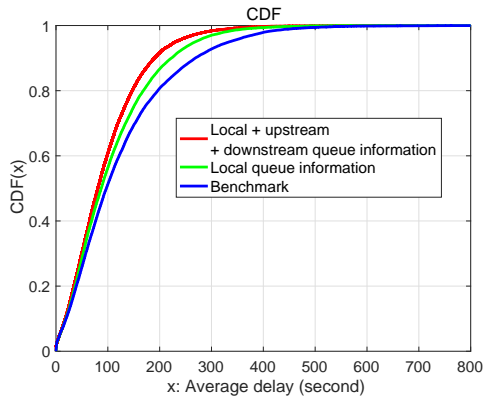}
\caption{The CDF of delay on two different queue management schemes.}
\label{cdf}
\end{figure}

Table~\ref{demandtable} provides another perspective on performance, using measurements of queue length and cluster size.
We list the measurements of all intersections whose average queue length is greater than $2$ vehicles in the high traffic demand scenarios. The average queue length and cluster size of the composite queue management scheme is less than those of the benchmark, and thus delay is lower according to Little's law \cite{little1961proof} of queueing theory. For the heaviest loaded intersections, e.g., Baum-Aiken and Centre-Aiken, the reduction of queue length could be up to $10$ times. In addition, the proposed queue management also leads to smaller variance of cluster size, which can be seen as the service time of queueing systems.  In the queueing literature, the variance of the service time plays a major role in queueing performance. For example, it is well-known that, for a $M/G/1$ queue, the variance in service time solely determines the average queueing delay if the average service time is kept the same. Similarly, even for $G/G/1$ queues, more `variable' service time leads to larger queueing delay. Although our system is far more complicated than these standard queueing systems, we observe that reducing variance of cluster size is still beneficial to scheduling systems.  

\begin{figure}[!htbp]
\centering
\includegraphics[scale = 0.28]{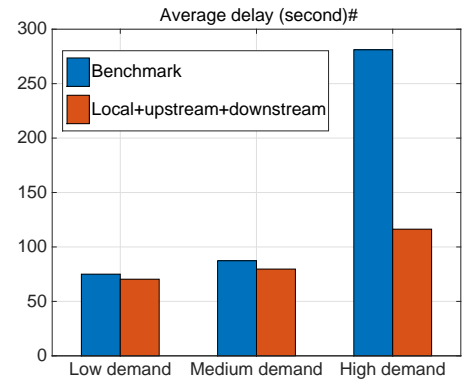}
\caption{Comparison of high, medium, low traffic demand.}
\label{three}
\end{figure}

In addition to overall performance, we differentiate the performance between the multinomial extension (\ref{multi}) of the multiple phases and the $2$-phased cases. The average delay and queue length between $2$-phased and $3$-phased intersections are plotted in Figure~\ref{multiple}. Figure~\ref{multiple}(a) shows that the multiple phase formulation improves delay by balancing the three different phases, and the performance is comparable to the original $2$-phased model. Also, the queue length is nearly half compared to the benchmark as shown in the Figure~\ref{multiple}(b). Since the realistic urban network is comprised of  intersections with different numbers of phases, this extension makes the original two-state Ising model fit in the real world. 

\section{Initial Field Experiment}
To further evaluate the performance potential of our proposed composite adaptive signal control system, an experiment was conducted in the field, using the $23$ intersections that make up the Baum-Centre neighborhood corridors. This experiment focused exclusively on traffic control performance during the PM Rush hour (4:00PM - 6:00PM), the heaviest traffic flow period of the day. It entailed first running the baseline schedule-driven approach to control these intersections for 3 consecutive weekdays (July 12 to 14, 2016) and then installing and running the proposed approach to control the same intersections for the PM Rush period on 3 consecutive weekdays of the following week (July 19 to 21, 2016). During both sets of days, queue-length and cluster size information was collected and computed from intersection sensor data.


 The queue-length information and cluster size were collected through use of intersection sensor data, which contains estimated queue-length and cluster size measurement. An analysis of those data showed that the proposed approach reduced the average queue length by $15\%$ and reduced the average cluster size by $12\%$ during the PM rush period, providing further evidence of the effectiveness of the proposed approach. Figures~\ref{init} shows the averaged queue and cluster information for both corridors and each corridor individually. 

\begin{figure}[!htbp]
\centering
\subfigure[Avg. delay]{\includegraphics[width=41.5mm, height=32mm]{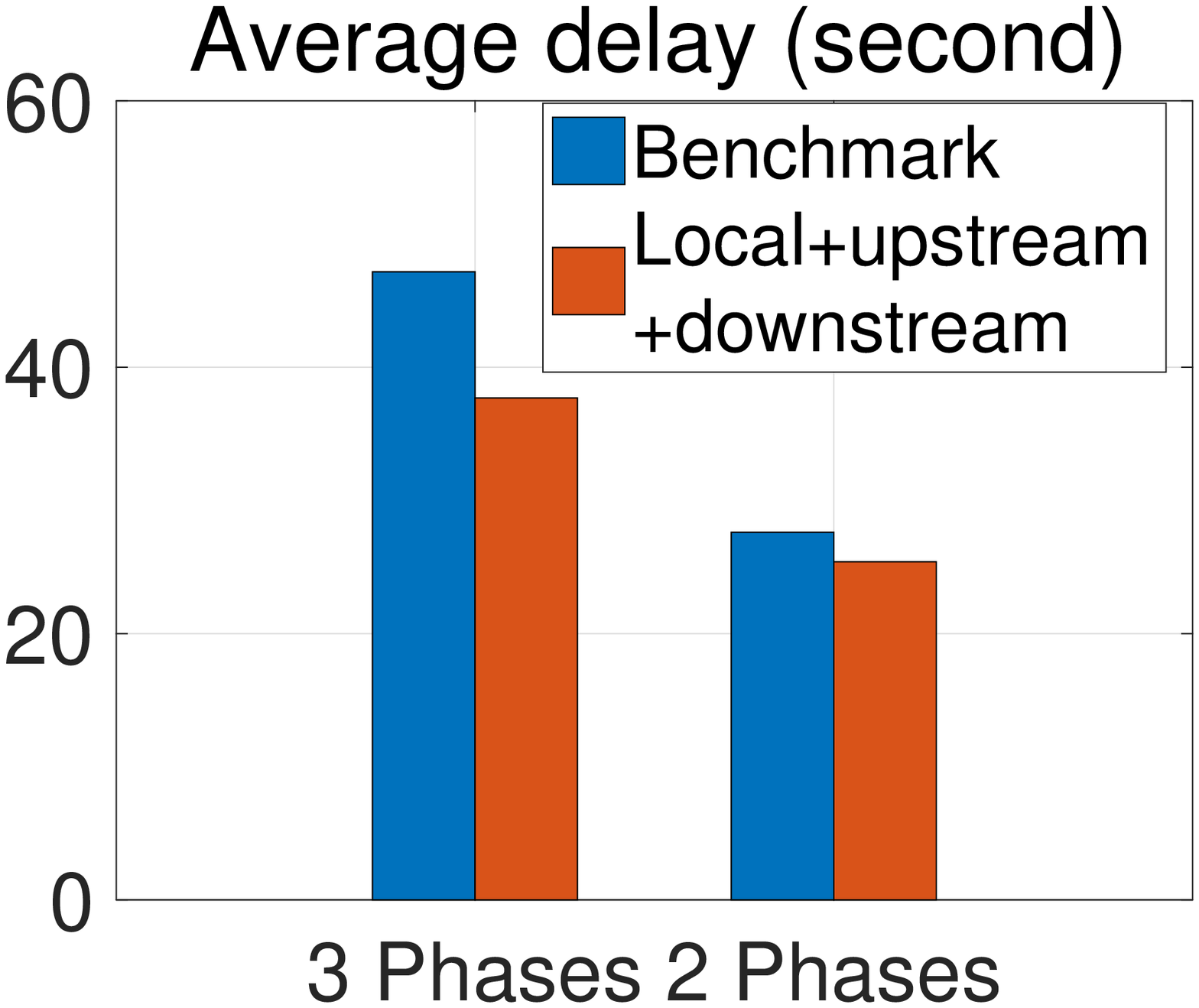}} 
\subfigure[Avg. queue]{\includegraphics[width=41.5mm, height=32mm]{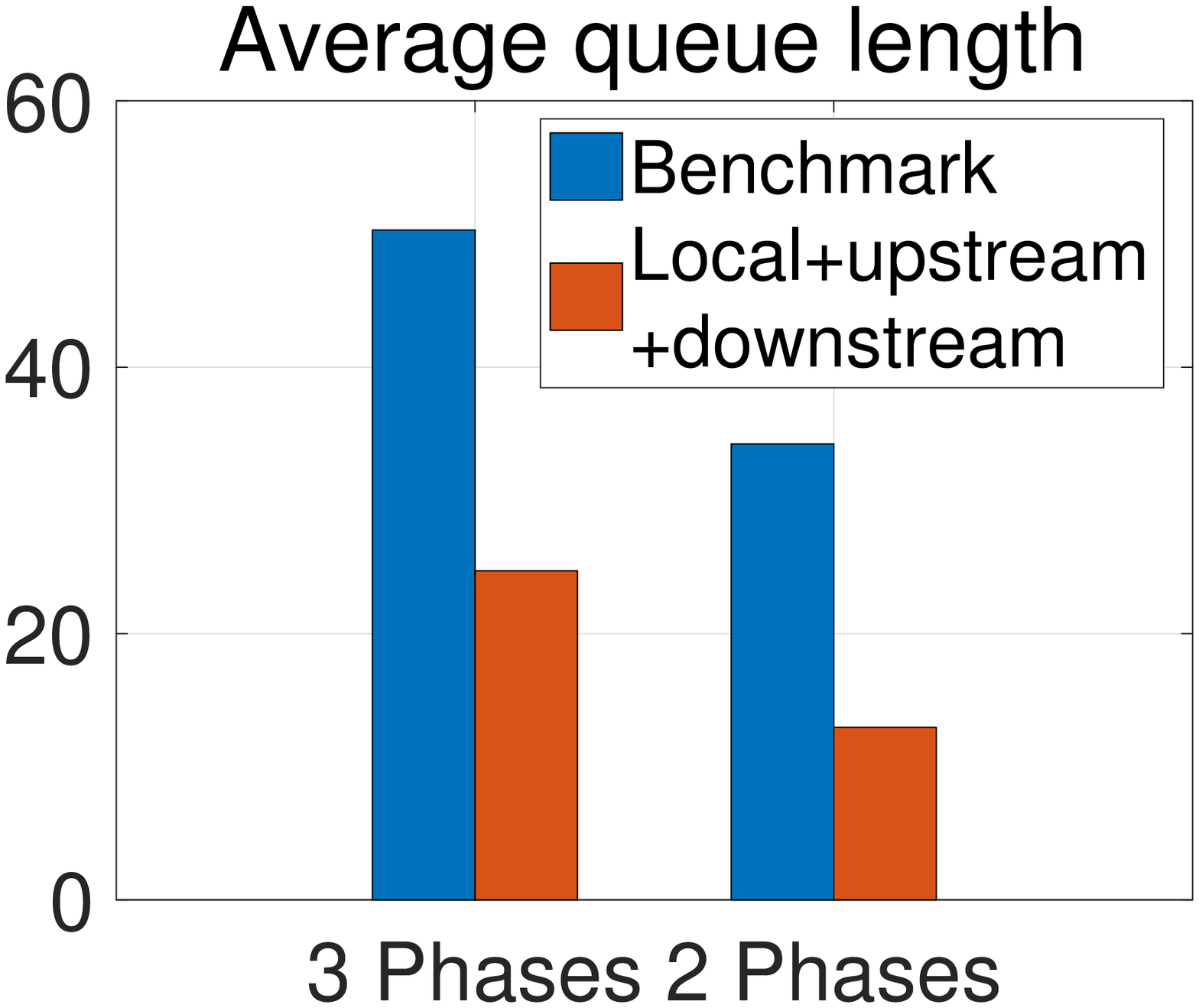}}
\caption{The comparison of average delay and queue length between $2$-phased and $3$-phased intersections}
\label{multiple}
\end{figure}


\section{Related Work}
Techniques that restrict queues from increasing in a network is a problem of broad interest in computer network design and manufacturing. Researchers from computer science, operation research, and communication engineering have been persistently working toward development of queue management techniques, which is able to stabilize queues and preserve performance. In the field of scheduling, however, this and related problems are rarely discussed. One exception is the recent research in coupling queueing theory concepts with finite capacity scheduling \cite{terekhov2014queueing,terekhov2014integrating}.

In \cite{tassiulas1992stability}, a ``backpressure'' policy called the max-weight algorithm was introduced to maximize the throughput of a network through stabilization of queues. This approach has been applied mainly to communication networks and but also recently to transportation networks \cite{wongpiromsarn2012distributed}.  However, there are two complications with applying backpressure to the problem of real-time traffic control. First, although backpressure is maximizing network throughput, the practical version \cite{tassiulas1998linear} still seems to induce large average delay \cite{shah2011hardness}, which is undesirable in the case of traffic networks. Second, the approach does not consider non-local influence from neighbors and is thus susceptible to myopic decisions. Actually, the proposed approach can be seen as a soft version of backpressure policy, so that the stability the queues is guaranteed. Furthermore, delay performance is not sacrificed due to the scheduling problem formulation. 

%
%

\begin{figure}[!htbp]
\centering
\subfigure[Avg. delay]{\includegraphics[width=41.5mm, height=32mm]{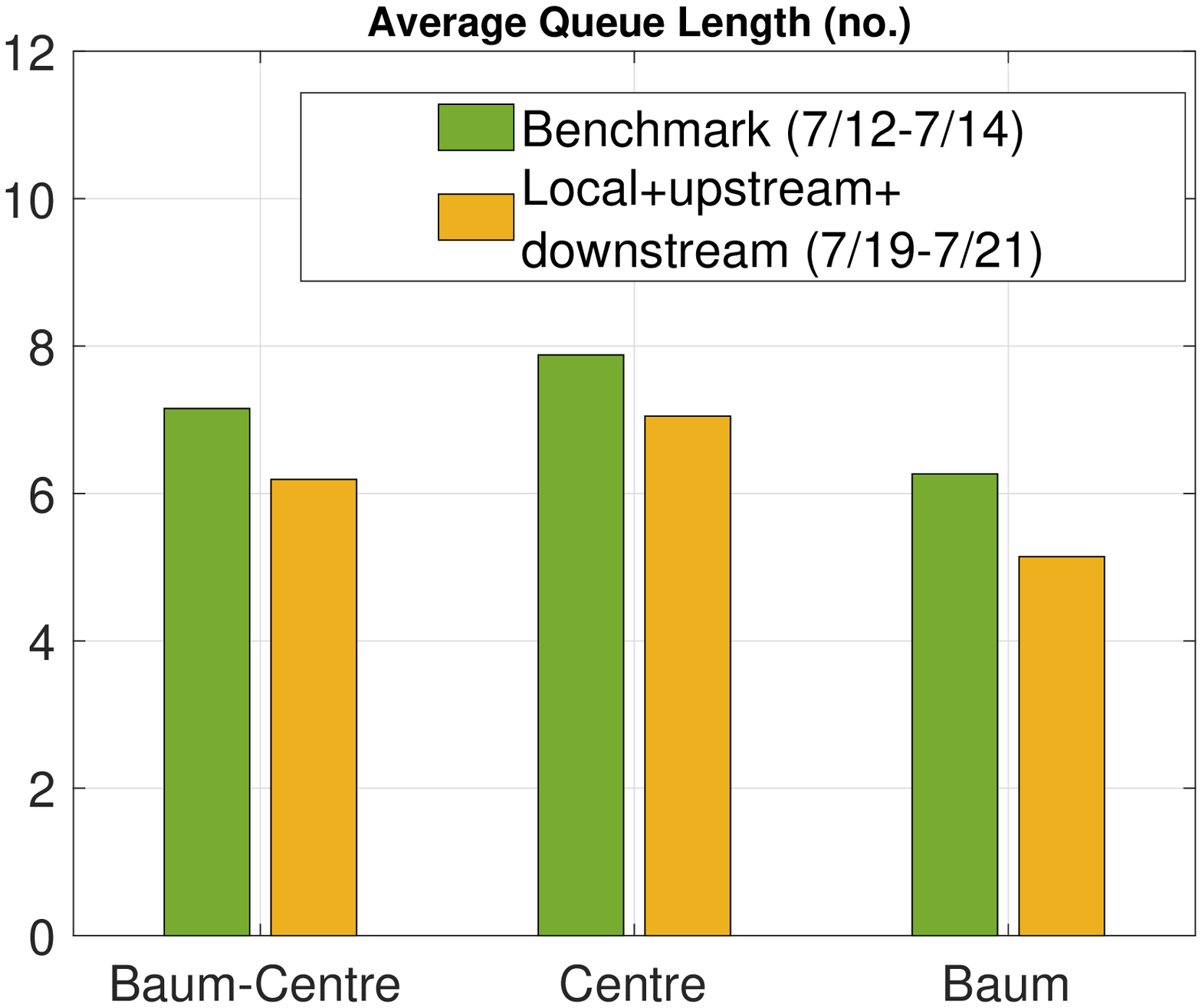}} 
\subfigure[Avg. queue]{\includegraphics[width=41.5mm, height=32mm]{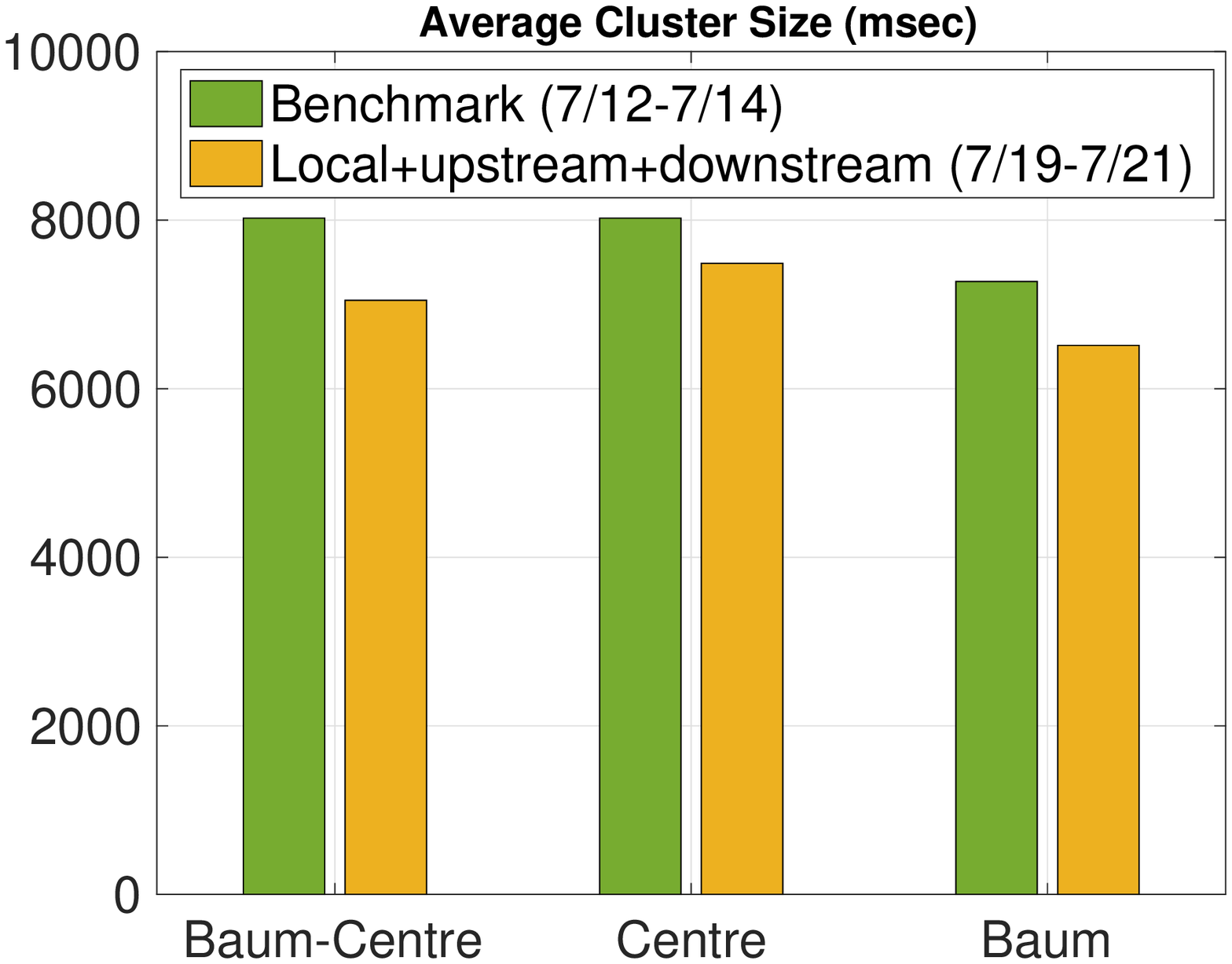}}
\caption{Initial field comparison on Baum-Centre corridors - average queue length and cluster size during PM rush hour.}
\label{init}
\end{figure}

Analytical network models based on queueing theory \cite{osorio2009surrogate} are another way to approach the problem of network congestion. By solving a large-scale optimization problem \cite{osorio2013simulation}, signal timing plans can be derived for an urban road network. As mentioned earlier, however, such an approach optimizes from a snapshot of average traffic flow, which is typically quite different than actual traffic flow. Traffic flow prediction work \cite{yin2002urban} has also concentrated in recent years on dealing with urban congestion. However, the interconnected queues increase the difficulty of predicting arrival rates of coming vehicles.

To our knowledge, statistical physics or graphical models have never been used to solve the network congestion problem.  Previous work has focused on using statistical physics to study traffic flow dynamics, although they have received less attention than other approaches. Recently, these approaches have witnessed a resurgence. For instance, \cite{jerath2015dynamic} recently studied the effect of driver algorithms on traffic flow. \cite{suzuki2013chaotic} applied the Ising model to study chaotic dynamics and serve as a starting point for considering statistical mechanics of traffic signals. In the context of self-organized traffic flow, however, these techniques have been less discussed due to the lack of suitable model.

\section{Conclusions}
In this work, we described a queue management scheme designed to gracefully boost the performance of schedule-driven traffic control as the level of congestion increases. The approach stabilizes the queues through exchange and use of queue-length information. This information is used to establish weights for clusters appearing in a given phase, using an Ising model of the intersection and ultimately applying mean field methods to compute the weights.
A decentralized message-passing protocol was developed and the composite system was evaluated on a simulation model of a real-world urban road network. Results showed that the composite queue management enhanced scheme improves average delay overall in comparison to the baseline schedule-driven traffic control approach, and that solutions provide substantial gain in highly congested scenarios. Future work will focus on the design of how to decompose the transportation network into several independent sets for approaching the optimality of network scheduling.

{\bf{Acknowledgement} }:This research was funded in part by the University Transportation Center on Technologies for Safe and Efficient Transportation at Carnegie Mellon University and the CMU Robotics Institute.


\bibliographystyle{aaai}
\bibliography{queue}
\end{document}